\documentclass[11pt]{article}
\usepackage[margin=1in]{geometry}
\usepackage[T1]{fontenc}
\usepackage{graphicx,xcolor,color}
\usepackage{amsfonts,amsmath,amssymb,amsthm}
\usepackage{hyperref}
\hypersetup{
	colorlinks=true,
	linkcolor=blue,
	citecolor=blue,
	filecolor=blue,      
}
\usepackage{color}

\usepackage{enumitem}
\setlist[itemize]{noitemsep}
\setlist{nosep,before=\vspace{0.5\baselineskip},after=\vspace{0.5\baselineskip}}
\providecommand{\keywords}[1]{\textbf{\textit{Keywords:}} #1}

\newtheorem{theorem}{Theorem}
\newtheorem{remark}{Remark}

\newcommand{\bA}{\mathbf{A}}
\newcommand{\bb}{\mathbf{b}}
\newcommand{\bp}{\mathbf{p}}

\newcommand{\bg}{\mathbf{g}}

\newcommand{\bx}{\mathbf{x}}
\newcommand{\by}{\mathbf{y}}
\newcommand{\bphi}{\boldsymbol{\phi}}
\newcommand{\bzeta}{\boldsymbol{\zeta}}

\DeclareMathOperator*{\argmin}{arg\,min}

\allowdisplaybreaks

\begin{document}
\title{Image Decomposition with G-norm Weighted by Total Symmetric Variation}

\author{Roy Y. He\thanks{City University of Hong Kong, Kowloon Tong, Hong Kong. Email: royhe2@cityu.edu.hk.}, 
Martin Huska\thanks{University of Bologna, Bologna, Italy. Email: martin.huska@unibo.it.}, 
Hao Liu\thanks{Hong Kong Baptist University, Kowloon Tong, Hong Kong.
Email: haoliu@hkbu.edu.hk.}
}

\date{}
\maketitle

\begin{abstract} 

In this paper, we propose a novel variational model for decomposing images into their respective cartoon and texture parts. Our model characterizes certain non-local features of any Bounded Variation (BV) image  by its \textit{total symmetric variation} (TSV). We demonstrate that TSV is effective in identifying regional boundaries. Based on this property, we introduce a weighted Meyer's $G$-norm to identify texture interiors without including contour edges. For BV images with bounded TSV, we show that the proposed model admits a solution.  Additionally, we  design a fast algorithm based on operator-splitting to tackle the associated non-convex optimization problem. The performance of our method is validated by a series of numerical experiments.

\keywords{Image decomposition, Meyer's $G$-norm, Operator-splitting.}
\end{abstract}

\section{Introduction}
Image decomposition refers to the process of representing an image as a combination of several images,  each of which ideally possesses distinct characteristics. It is closely linked to our  ability to swiftly distinguish between smooth surfaces and rough walls,  soft shadows and sharp edges, as well as repetitive patterns and random noise. Such phenomenon has led to extensive research into variational models for image decomposition based on image gradient~\cite{aubert2005modeling,aujol2006structure}. One of the most famous examples is the Rudin-Osher-Fatemi (ROF)  total variation (TV) model~\cite{rudin1992nonlinear} for decomposing an $L^2$ image into its cartoon  and noise part:
\begin{equation}\label{eq_model_ROF}
\min_{\substack{f=u+v, \\
u\in\text{BV}(\Omega), v\in L^2(\Omega)}}\int_{\Omega}|D u| + \mu \|v\|_{L^2(\Omega)}\;.
\end{equation}
Here $\int_{\Omega}|Du|$ is the TV-norm of $u$ and $\mu>0$ is the regularization parameter. Meyer~\cite{meyer2001oscillating}   suggested  the $G$-norm, which was generalized by Aubert and Aujol~\cite{aubert2005modeling} to bounded regions,  to model the oscillatory part. The Meyer's model is
\begin{equation}\label{eq_model_Meyer}
\min_{\substack{f=u+v, \\
u\in\text{BV}(\Omega), v\in G(\Omega)}}\int_{\Omega}|D u| + \mu \|v\|_{G(\Omega)}\;,
\end{equation}
whose dual relation with~\eqref{eq_model_ROF} is discussed in~\cite{meyer2001oscillating}. Diverse directions have been widely explored including capturing oscillations with other norms~\cite{osher2003image,aujol2006structure}, modeling structure by non-convex term \cite{Huska_2019}, and decomposition into more components~\cite{huska2021variational,he2024euler}. Most of them focus on low-level contrast  and arbitrary oscillations.

Textures are specific regional oscillations that have more statistical regularity. Analogous to edges, they  contribute to early-stage segmentation, and many psycho-physical studies have found the connection between the distribution of non-local contrast and pre-attentive segregation of textures~\cite{julesz1973inability}. For example, Julesz et al.~\cite{julesz1973inability} suggested the \textit{dipole statistics} which consider image gray levels at opposite ends of line segments with varying lengths, orientations, and placements.  
Other models involve size~\cite{olson1970variables}, contour termination~\cite{marr1976early}, and line crossings~\cite{bergen1988early}. Arguably, these features cannot be fully characterized by magnitudes of intensity gradient. We need to consider the contextual influences~\cite{grigorescu2004contour} that lead to differentiating isolated edges from grouped edges. Some statistical solutions including Histogram of Oriented Gradient (HoG)~\cite{zuo2014gradient}, spectral histogram~\cite{liu2003texture} and neural network-based method~\cite{yang2020learning} are available.

In this paper, we note that the non-local variant of image gradient~\cite{gilboa2009nonlocal} suggests an interesting direction to tackle this problem. Consider a compact Lipschitzian domain $\Omega\subset\mathbb{R}^2$, an image $f\in C^1_0(\Omega,\mathbb{R})$, and a nonnegative symmetric function $w:\Omega\times\Omega\to\mathbb{R}$. Define the non-local gradient of $f$ at $\bx\in\Omega$ as
\begin{equation}
\nabla_w f(\bx,\by) = (f(\by)-f(\bx))\sqrt{w(\bx,\by)}\;,~\forall\by\in\Omega\;.\label{eq_NLG}
\end{equation}
In practice, $w$ often has bounded impacts, i.e., $|w(\bx,\by)|=0$ if $\|\bx-\by\|_2>r$ for some small $r>0$. In case $w(\bx,\by)=\chi(\|\bx-\by\|_2<r)$ for some $r>0$, we recover the  dipole statistics~\cite{julesz1973inability}.  Here $\chi$ denotes the characteristic function giving a value of $1$ if the condition argument is true; and $0$ otherwise.

We observe that~\eqref{eq_NLG} can be expressed in a line integral form 
\begin{equation}
\nabla_wf(\bx,\by) = \int_{\mathbb{R}}f(\bx+t(\by-\bx))\mu(t,\bx,\by)\,dt\;,\label{eq_NLG1}
\end{equation}
where $\mu(t,\bx,\by)= \sqrt{w(\bx,\bx+t(\by-\bx))}H'(t)$; $H(t)=\chi(0\leq t\leq 1)$, and its derivative is understood in the distributional sense. Formally, we have
\begin{equation}
\nabla_wf(\bx,\by) = -\int_{\mathbb{R}}W(t,\bx,\by)\frac{d f}{dt}(\bx+t(\by-\bx))\,dt\;,\label{eq_NLG2}
\end{equation}
where $W$ is a (generalized) function whose distributional derivative with respect to $t$ is $\mu(t,\bx,\by)$. The expression~\eqref{eq_NLG2} implies that $\nabla_wf(\bx,\by)$ can be regarded as a weighted integration of the directional derivative of $f$. In particular, if we impose certain conditions on $W$, we shall show that~\eqref{eq_NLG2} provides a way to quantify the incoherent variations of $f$ along one direction. Since image gradient can be sensitive to noise, denoising is necessary for noisy inputs to ensure the effectiveness of this approach. In this work, we focus on clean images.

\begin{figure}[t!]
    \centering
    \begin{tabular}{ccc}
    \includegraphics[width=0.3\linewidth]{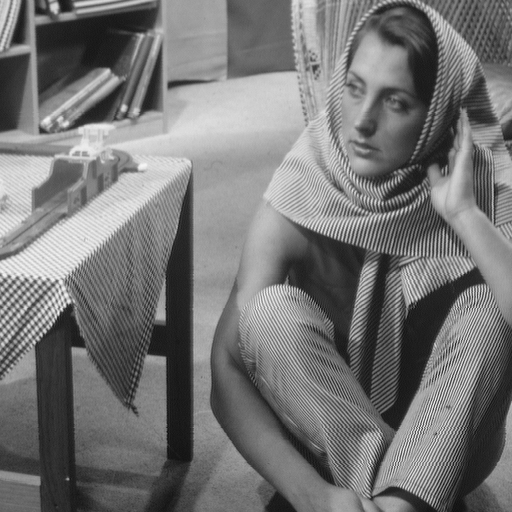} & \includegraphics[width=0.3\linewidth]{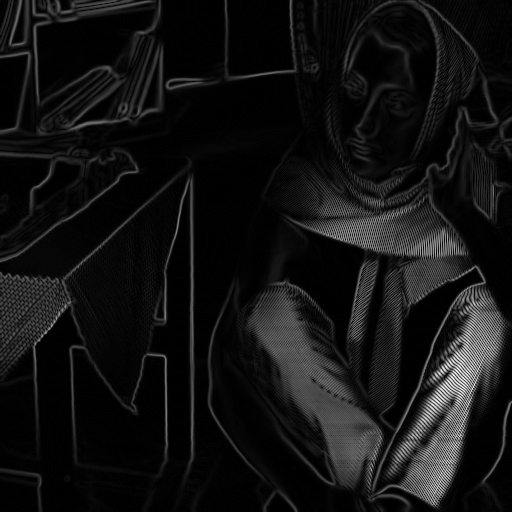} & \includegraphics[width=0.3\linewidth]{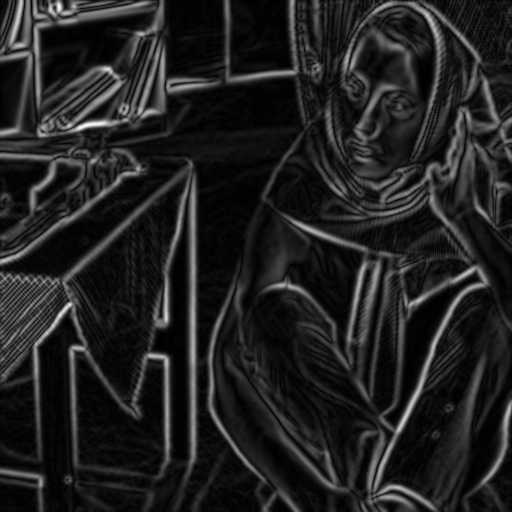}\\
            \includegraphics[width=0.3\linewidth]{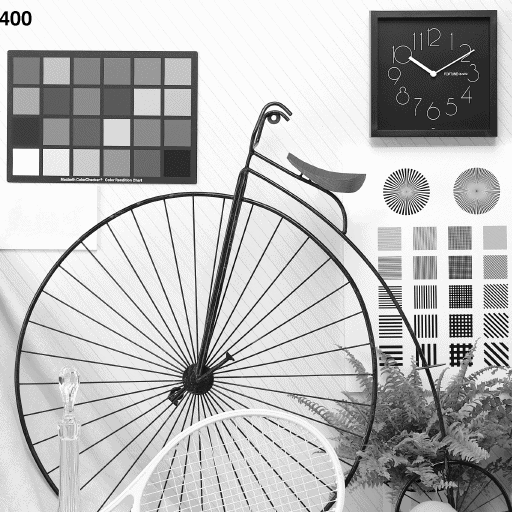} & \includegraphics[width=0.3\linewidth]{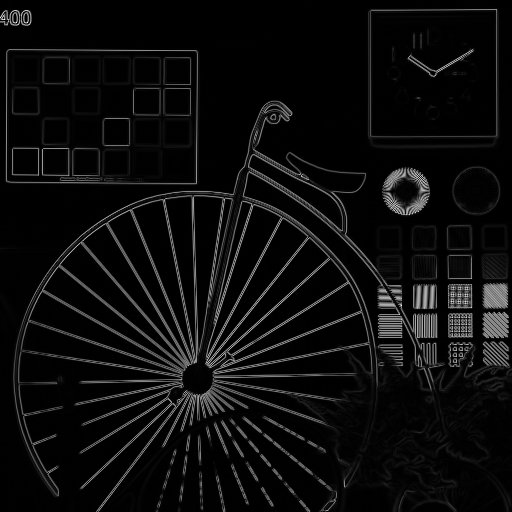} & \includegraphics[width=0.3\linewidth]{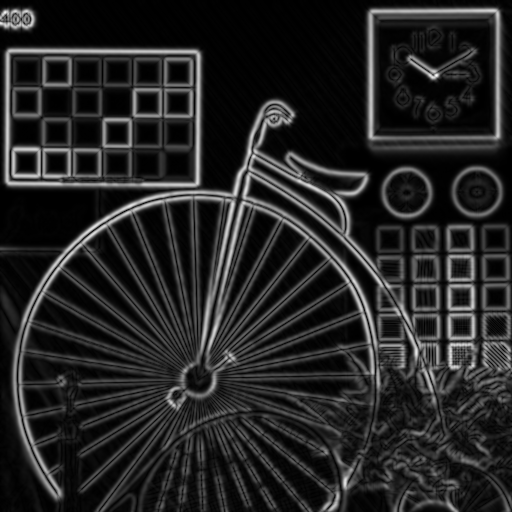}\\
        (a)&(b)&(c)\\
    \end{tabular}
    \caption{Image gradient (b) and proposed Total Symmetric Variation (TSV)~\eqref{eq_TSV} (c), computed for (a). TSV is high at regional boundaries while remains low in homogeneous and textural interiors. We exploit this property for image decomposition.}
    \label{fig:TSV_show}
\end{figure}

Our contributions are: \textbf{(I)} We propose \textbf{T}otal \textbf{S}ymmetric \textbf{V}ariation (TSV) generalizing  the observation above. We formulate~\eqref{eq_NLG2} for images with bounded variations and extend $W$ to be any radially symmetric  functions. TSV allows us to detect \textit{regional boundaries} including those between homogeneous regions or those between regions with different textures. See Figure~\ref{fig:TSV_show}.  \textbf{(II)} Leveraging the TSV of an image $f$, we  propose an adaptive variational model for cartoon-texture decomposition.
We leverage the TSV of the input image $f$ to resolve the ambiguity between structural and textural edges. \textbf{(III)} We design an effective operator-splitting algorithm to tackle the associated non-convex problem.

We organize the paper as follows. In Section~\ref{sec_proposed}, we introduce our model for image decomposition. More specifically, we describe the model and  weighted $G$-norm with general weight functions in Section~\ref{subsec_weight_G_model}; we propose the total symmetric variation (TSV) in Section~\ref{sec_TSV}; and present the variational form of our model in Section~\ref{subsec_variation_model}. In Section~\ref{sec_algorithm}, we develop an operator-splitting algorithm for the proposed model. In Section~\ref{sec_experiment}, we  illustrate the effectiveness of our method through  numerical ablation and comparison studies. We conclude  in Section~\ref{sec_conclude}.

\section{Proposed Method}\label{sec_proposed}
Given a noise-free image $f:\Omega\to\mathbb{R}$ defined over a rectangular  domain $\Omega\subset\mathbb{R}^2$, we  decompose  $f$ into its cartoon or structural part and its textural part. 

\subsection{Non-local weighted  model}\label{subsec_weight_G_model}
We propose the following variational model for decomposition:
\begin{align}
\min_{\substack{u,v,\ u+v=f}}\int_{\Omega}|D u|+\lambda \|\eta v\|_G.
\label{eq.model.0}
\end{align}
The image $u$ represents the structural part and $v$ represents the textural part of $f$, respectively. The energy functional consists of the TV-norm of $u$ and the Meyer's $G$-norm of a weighted $v$, which is controlled by a regularization parameter $\lambda>0$. The  weight function $\eta:\Omega\to (0,+\infty)$ depends on $f$ and is an important part of our model. It incorporates non-local features of $f$ and suppresses  regional boundaries. We note that without the weight, $G$-norm alone may not distinguish between oscillatory regions and edges. We expect that $\eta(\bx)$ is large when $\bx$ is close to regional boundaries of $f$; and $\eta(\bx)$ is small when $\bx$ is inside homogeneous or textural regions of $f$.
Many candidates exist to satisfy these properties, and we will propose the \textit{total symmetric variation} in Section~\ref{sec_TSV} to build our weight function. The following discussion applies to general weights.

 Recall that the Meyer's $G$-norm of $v\in L^2(\Omega)$  is defined by~\cite{aubert2005modeling}
\begin{equation}
\|v\|_{G(\Omega)} = \inf\{\|\bg\|_{L^\infty(\Omega,\mathbb{R}^2)}~|~v=\nabla\cdot\,\bg\;,\bg\cdot \mathbf{N}=0~\text{on}~\partial\Omega \}\;.
\end{equation}
where $\mathbf{N}$ is the outward normal direction of $\partial\Omega$. It is thus natural to interpret
\begin{equation}\label{eq_weightG_norm}
\| \eta v\|_{G(\Omega)} = \inf\{\|\bg\|_{L^\infty(\Omega,\mathbb{R}^2)}~|~\eta v=\nabla\cdot\,\bg\;,\bg\cdot \mathbf{N}=0~\text{on}~\partial\Omega \}\;.
\end{equation}
 The weighted semi-norm~\eqref{eq_weightG_norm} induces a weighted space
\begin{equation}\label{eq_weighted_Gspace}
G_{\eta}(\Omega) = \{v\in L^2(\Omega)~|~ \|\eta v\|_{G(\Omega)}<+\infty\}\;,
\end{equation}
and we denote $\|v\|_{G_{\eta}(\Omega)}=\|\eta v\|_{G(\Omega)}$. We remark that the notion of weighted $G$-norm was considered in \cite{cui2023surface} for surface reconstruction from noisy point clouds. In this paper, it is the first time that it is used to distinguish regional boundaries and textural regions for image decomposition.

\subsection{Total symmetric variation of BV functions}\label{sec_TSV}
In this section, we propose a weight function $\eta$ in~\eqref{eq.model.0} to highlight regional boundaries. 
Recall that the \textit{total variation} of $f\in L^1(\Omega)$ is defined by
\begin{equation}
\int_{\Omega}|D f| = \sup\bigg\{\int_{\Omega} f(\bx)\nabla\cdot\bphi(\bx)\,d\bx,~\bphi\in C_0^\infty(\Omega,\mathbb{R}^2), \|\bphi\|_{L^\infty}\leq 1\bigg\}\;,
\end{equation}
and $\text{BV}(\Omega)=\{f\in L^1(\Omega)~|~\int_{\Omega}|Df|<+\infty\}$. When $f\in \text{BV}(\Omega)$, it is known that its one-dimensional slices are also of bounded variation~(\cite{ambrosio2000functions}, Remark 3.104). More precisely,  for any $\bzeta\in \mathbb{S}^1$, $f_{\by,\bzeta}(t) := f(\by + t\bzeta)\in \text{BV}(\Omega_{\by}^{\bzeta})$ for $\mathcal{H}^1$-almost every (-a.e.) $\by\in L_{\bzeta}$, which is a line perpendicular to $\bzeta$ passing through the origin,  $\Omega_{\by}^{\bzeta}:=\{\by+s\bzeta\in\Omega: s\in\mathbb{R}\}$ is a line segment in $\Omega$,   and $\mathcal{H}^s$ is  the $s$-dimensional Hausdorff measure. We can thus  define  for $\mathcal{H}^2$-almost every $\bx\in\Omega$ a Radon measure $Df_{\bx,\bzeta}$  along the line passing through $\bx$ in the direction of $\bzeta$. 

Let $B(\mathbb{R})$ be the set of bounded, symmetric, Borel measurable functions on the real line. Define the following set of functions with compact support:
\begin{equation}
\mathcal{W}=\{w\in B(\mathbb{R}):w(x)=w(-x), x\in\mathbb{R}~\text{and}~ w(x)=0, |x|>r~\text{for some } r>0\}\;.
\end{equation}
 For any $w\in \mathcal{W}$ with support $[-r,r]$, we define  the \textit{total symmetric  variation} (TSV) of $f\in\text{BV}(\Omega)$ associated with $w$ by
\begin{equation}
\left|D_wf\right|(\bx):=\int_{S^1}\left|\int_{\Omega_{\bx}^{\bzeta}}w\,dDf_{\bx,\bzeta}\right|\,d\mathcal{H}^{1}(\bzeta)\;,~\text{for}~\mathcal{H}^2\text{-a.e.}~\bx\in\Omega\;,\label{eq_TSV}
\end{equation}
 and set $|D_wf|(\bx)=+\infty$ for the other $\bx$ in a $\mathcal{H}^2$-zero set. As~\eqref{eq_NLG2} suggests, the TSV computes the weighted integration of the directional derivative along $\Omega_{\bx}^{\bzeta}$.

 The proposed TSV~\eqref{eq_TSV} considers symmetric variations of $f$ at $\bx$ in all directions. Specifically, the absolute integral $I_{w}^{\bzeta}(x)=\left|\int_{\Omega_{\bx}^{\bzeta}}w\,dDf_{\bx,\bzeta}\right|$ measures the visual incoherence of $f$ caused by asymmetric local contrast changes along the direction of $\bzeta$. If the variation of $f$ along $\bzeta$ on one side of $\bx$ cancel with the variation on the other side, then $I_{w}^{\bzeta}(x)$ is small; whereas if the contrast of $f$ varies inconsistently around $\bx$ along the direction of $\bzeta$, then $I_{w}^{\bzeta}(x)$ becomes large. Typically, when $\bx$ is near a boundary between two homogeneous regions or  regions with different textural patterns, $I_{w}^{\bzeta}(x)$ is large for $\mathcal{H}^1$-almost all directions $\bzeta$. 
 
\begin{remark}
We note that the proposed TSV~\eqref{eq_TSV} is related to various models  in the literature for understanding texture segregation. For example, when $w(x)=\chi([-r,r])$ for finite $r>0$, TSV is an extension of the  \textit{dipole statistics} suggested in~\cite{julesz1973inability}. An anisotropic variant of TSV~\eqref{eq_TSV} can be defined if $w$ also depends on the direction $\bzeta\in \mathbb{S}^1$. 
Different from the \textit{windowed inherent variation}~\cite{xu2012structure}, which is defined by locally integrating absolute values of image gradients within rectangular windows, the proposed TSV~\eqref{eq_TSV} evaluates line integrals of absolute values of image gradients along all directions emanating from single points.
\end{remark}

\subsection{Variational formulation and existence of solutions}\label{subsec_variation_model}
We introduce the variational formulation of our model as follows. Given $f\in \text{BV}(\Omega)$ and $w\in\mathcal{W}$ satisfying  $|D_wf|\in L^\infty(\Omega)$, we decompose $f$ into its structure part $u$ and texture part $v$ by solving the minimization problem:

\begin{equation}\label{var_model}
\min_{\substack{u\in\text{BV}(\Omega),v\in G_{\eta}(\Omega)\\u+v=f}} \int_{\Omega}|D u|+\lambda\|v\|_{G_{\eta}(\Omega)},~\text{where}~\eta(f) = \kappa + |D_wf|
\;.\end{equation}
In~\eqref{var_model} above, the weight function $\eta$ satisfies the desired properties stated in Section~\ref{subsec_weight_G_model}, and $\kappa>0$ is a constant to enforce strong positivity. 
Analogously to the existence result for the Meyer's model~\cite{meyer2001oscillating}, we can prove that
\begin{theorem}Suppose that  $f\in \text{BV}(\Omega)$, $w\in\mathcal{W}$ satisfying  $|D_wf|\in L^\infty(\Omega)$, and $\kappa>0$, then the functional in~\eqref{var_model} admits a minimizer.
\end{theorem}
\begin{proof} Consider  a minimizing sequence $\{v_n\}\subset G_{\eta}(\Omega)$ for $\mathcal{E}(v):=\int_{\Omega}|D (f-v)|+\lambda\|v\|_{G_{\eta}(\Omega)}$. Since $\mathcal{E}(0) = \alpha_1\int_{\Omega}|D f|<+\infty$, we know that $\{v_n\}$ is bounded in $\text{BV}(\Omega)$, which is continuously embedded in $L^2(\Omega)$. Therefore, $\{v_n\}$ has a weakly convergent subsequence in $L^2(\Omega)$, still denoted as $\{v_n\}$, i.e.,  there exists some $v^*\in L^2(\Omega)$ such that $v_n\overset{L^2(\Omega)}{\rightharpoonup} v^*$. Since by Cauchy-Schwarz inequality
$$\left|\int_{\Omega}\eta v^*\,d\bx\right|=\left|\int_{\Omega}\eta (v^*-v_n)\,d\bx\right|\leq \|\eta\|_{L^2(\Omega)}\|v^*-v_n\|_{L^2(\Omega)},$$  
we see that $\int_{\Omega}\eta v^*\,d\bx=0$.  By~\cite{bourgain2003equation}, there exists $\bg^*\in L^\infty(\Omega,\mathbb{R}^2)$ such that $\nabla\cdot\,\bg^* = \eta v^*$ and $\bg^*=\mathbf{0}$ on $\partial\Omega$; hence,  $v^*\in G_{\eta}(\Omega)$. 
\end{proof}
Depending on $f$ and $w$, the uniqueness of solution is not guaranteed in general.

\section{Proposed Operator-Splitting Algorithm}\label{sec_algorithm}
In this section, we design a fast iterative algorithm based on operator-splitting to address the proposed model~\eqref{var_model}. The  computational complexity per iteration is in the order of $\mathcal{O}(N\ln N)$ for an image with $N$ pixels.

\subsection{Problem reformulation}
We reformulate~\eqref{var_model}  by adopting the barrier method for the constraint  and approximating the $G$-norm via the inverse Sobolev norm~\cite{osher2003image,huska2021variational}. In addition,  we find that adding a regularization parameter for the TV term results in more robust performances. As a result, we are led to
\begin{align}
&\min_{u,\bg}\alpha_1\int_{\Omega}|\nabla u|d\bx + \alpha_2 \int_{\Omega}|\bg|^2\,d\bx+\frac{1}{2\theta}\int_{\Omega}\left(u+\frac{1}{\eta}\nabla\cdot \bg-f\right)^2\,d\bx,
\label{eq.model.new}
\end{align}
where $\alpha_1>0$ and $\alpha_2>0$ are regularization parameters. Note that when $\theta\to 0^+$,~\eqref{eq.model.new} is equivalent to~\eqref{var_model} for $\lambda = \alpha_2/\alpha_1$, see~\cite{aujol2006color}. We then introduce an auxiliary variable $\bp$ for $\nabla u$ and define the set 
\begin{align*}
    &S_1=\{(\bg,v)\in (H^1(\Omega))^2\times L^2(\Omega): v=\frac{1}{\eta}\nabla\cdot \bg\},\\
    &S_2=\{(\bp,v)\in (H^1(\Omega))^2\times L^2(\Omega): \exists u\in H^1(\Omega) \mbox{ s.t. } \bp=\nabla u,\int_{\Omega}u+v d\bx=\int_{\Omega} fd\bx\}
\end{align*}
and their corresponding indicator function $I_{S_k}(\cdot)$ for $k=1,2$, which equals to $0$ if its argument is in $S_k$, and $+\infty$ otherwise.
Note that $S_2$ is not empty. Given $\bp$ and $v$, there exists a unique $u_{\bp,v}\in H^1(\Omega)$ satisfying the zero-Neumann boundary condition. See~\cite{deng2019new} for example. Denote the following convex functionals
\begin{align*}
    &J_1(\bp)=\alpha_1\int_{\Omega} |\bp|d\bx,\;J_2(\bg)=\alpha_2\int_{\Omega} |\bg|^2d\bx,\;J_3(\bp,v)=\frac{1}{2\theta} \int_{\Omega} (u_{\bp,v}+v-f)^2d\bx,
\end{align*}
then~\eqref{eq.model.new} is equivalent to
\begin{align}
    \min_{\substack{v\in L^2(\Omega),\\
    \bp\in (H^1(\Omega))^2,\bg\in (H^1(\Omega))^2}} J_1(\bp)+ J_2(\bg) + J_3(\bp,v) + I_{S_1}(\bg,v) + I_{S_2}(\bp,v).
    \label{eq.model.aux}
\end{align}
If $(\bp^*,\bg^*,v^*)$ is a solution to (\ref{eq.model.aux}), then $(u_{\bp^*,v^*},\bg^*)$ solves (\ref{eq.model.new}). Moreover, $(\bp^*,\bg^*,v^*)$ 
is a  stationary solution to the following initial value problem
\begin{align}
    \begin{cases}
        \frac{\partial \bp}{\partial t} + \partial_{\bp} J_1(\bp)  + \partial_{\bp} J_3(\bp,v) + \partial_{\bp} I_{S_2}(\bp,v) \ni \mathbf{0},\\
        \frac{\partial \bg}{\partial t} + \partial_{\bg} J_2(\bg)  + \partial_{\bg} I_{S_1}(\bg,v) \ni \mathbf{0},\\
         \frac{\partial v}{\partial t} + \partial_v J_3(\bp,v)  + \partial_{v} I_{S_1}(\bg,v) + \partial_{v} I_{S_2}(\bp,v) \ni 0,
    \end{cases}
    \label{eq.ivp}
\end{align}
where $\partial$ denotes the subgradient. 
We use operator-splitting methods to solve (\ref{eq.ivp}). Alternatively, one can use alternating direction method of multipliers (ADMM) to solve (\ref{eq.model.new}). Compared to operator-splitting methods, ADMM is more sensitive to parameters. And as reported in \cite{deng2019new}, operator-splitting methods are more efficient and provide results with visually better effects.
\subsection{Operator splitting scheme}
Denote $t^n=n\Delta t$ with $\Delta t $ being the time step. We denote the numerical solution $(\bp,\bg,v)$ at $t^n$ by $(\bp^n,\bg^n,v^n)$. Given $(\bp^n,\bg^n)$, we use the Lie scheme~\cite{deng2019new} with one-step backward Euler method for time discretization to update $(\bp^n,\bg^n,v^n)\rightarrow (\bp^{n+1/2},\bg^{n+1/2},v^{n+1/2})\rightarrow (\bp^{n+1},\bg^{n+1},v^{n+1})$ via two substeps, leading to a time discretization scheme of the Marchuk--Yanenko type \cite{glowinski2016some}:
\begin{align}
&\begin{cases}
     \frac{\bp^{n+1/2}-\bp^n}{\Delta t}+ \partial_{\bp}J_1(\bp^{n+1/2})\ni \mathbf{0},\\
     \frac{\bg^{n+1/2}-\bg^n}{\Delta t}+ \partial_{\bg} J_2(\bg^{n+1/2})+ \partial_{\bg} I_{S_1}(\bg^{n+1/2},v^{n+1/2}) \ni \mathbf{0},\\
     \frac{v^{n+1/2}-v^n}{\Delta t} + \partial_{v} I_{S_1}(\bg^{n+1/2},v^{n+1/2}) \ni 0,
    \end{cases}
    \label{eq.split.dis.1}\\
   & \begin{cases}
        \frac{\bp^{n+1}-\bp^{n+1/2}}{\Delta t} + \partial_{\bp}J_3(\bp^{n+1},\bg^{n+1}) + \partial_{\bp} I_{S_2}(\bp^{n+1},\bg^{n+1}) \ni \mathbf{0},\\
        \frac{\bg^{n+1}- \bg^{n+1/2}}{\Delta t}=\mathbf{0},\\
        \frac{v^{n+1}- v^{n+1/2}}{\Delta t} +  \partial_{v}J_3(\bp^{n+1},v^{n+1}) + \partial_{v} I_{S_2}(\bp^{n+1},v^{n+1}) \ni 0.
    \end{cases}
    \label{eq.split.dis.2}
\end{align}

The rest of this section discusses solutions to each subproblem.
Let $u$ be an image of size $M\times N$. Considering periodic boundary condition, we define the discretized difference operator
\begin{align*}
    &\partial^{\pm}_1 u(i,j)= \pm(u(i\pm 1,j)-u(i,j)), \ \partial^{\pm}_2 u(i,j)= \pm(u(i,j\pm1)-u(i,j)),\\
    &\partial^{\pm}_{3} u(i,j)= \pm(u(i\pm 1,j\pm1)-u(i\pm1,j\pm1)),\\
    &\partial^{\pm}_{4} u(i,j)= \pm(u(i\pm 1,j\mp1)-u(i\pm1,j\mp1)).
\end{align*}
Let $\mathfrak{W}(i,j)$ be a window centered at $(i,j)$. Using line integral in the  TSV~\eqref{eq_TSV} many cause instability. Here we use a rotated anisotropic Gaussian defined over $\mathfrak{W}(i,j)$ to approximate it. Specifically, we define $w_{\theta}(i,j)=Ce^{-(ai^2+2bij+cj^2})$ for
\begin{align}
    a=\frac{\cos^2(\theta)}{2\sigma_1}+ \frac{\sin^2(\theta)}{2\sigma_2}, \ b= \frac{\sin(2\theta)}{4\sigma_1}-\frac{\sin(2\theta)}{4\sigma_2}, c=\frac{\sin^2(\theta)}{2\sigma_1}+ \frac{\cos^2(\theta)}{2\sigma_2},
    \label{eq:abcd}
\end{align}
where $C$ is a normalizing parameter so that $\sum_{(k,\ell)\in \mathfrak{W}(i,j)} w(k-i,\ell-j)=1$.
Let $\{\theta_{\gamma}\}_{\gamma=1}^4=\{0, \pi/2, \pi/4, 3\pi/4\}$ be the rotating directions corresponding to the difference operator defined above. Our discretized approximation of TSV with four directions is
\begin{align}
    |D_wf|(i,j)= \sum_{\gamma=1}^4\left|\sum_{(k,\ell)\in  \mathfrak{W}(i,j)} w_{\theta_{\gamma}}(k-i,\ell-j)\partial^+_{\gamma} u(k,\ell)\right|.
    \label{eq:Dwf_disc}
\end{align}
The discretization \eqref{eq:abcd}--\eqref{eq:Dwf_disc} allows to control the line integral shape not by changing the size of $\mathfrak{W}(i,j)$, instead using $\sigma_1$ to control the line length, while $\sigma_2$ controls the line width.

\subsection{On the solution to (\ref{eq.split.dis.1})}
In (\ref{eq.split.dis.1}), $\bp^{n+1/2}$ solves
\begin{align}
    \bp^{n+1/2}=\argmin_{\bp} \left[\frac{1}{2}\int_{\Omega} |\bp-\bp^n |_2^2dx + \Delta t \alpha_1\int_{\Omega} |\bp|dx\right],
\end{align}
and we have a closed-form expression $\bp^{n+1/2}=\max\left\{ 0, 1-\frac{\Delta t \alpha_1}{|\bp^n|}\right\} \bp^n.$
For $\bg^{n+1/2}$, we solve for it by the following second-order equation:
\begin{align}
     \bg^{n+1/2} -& c\nabla(\nabla\cdot \bg^{n+1/2}) + 2\Delta t \alpha_2 \bg^{n+1/2}\nonumber \\
    &= \bg^n -c\nabla(\nabla\cdot \bg^{n}) +  \nabla \left(\frac{1}{\eta^2} \nabla\cdot \bg^{n}\right) -\nabla\left(\frac{1}{\eta} v^n\right),\label{eq_frozen}
\end{align}
where $c>0$ is a frozen coefficient~\cite{deng2019new}.
This leads to $\bA\bg^{n+1/2}=\bb$ with
\begin{align*}
    &\bA=\begin{bmatrix}
        -c\partial_1^+\partial_1^- +2\Delta t\alpha_2 & -c\partial_1^+\partial_2^-\\
        -c\partial_2^+\partial_1^- & -c\partial_2^+\partial_2^- +2\Delta t\alpha_2
    \end{bmatrix},~\text{and}\\
    &\bb=\begin{bmatrix}
         g_1^n-\partial_1^+((c - \frac{1}{\eta^2})(\partial_1^- g_1^n + \partial_2^- g_2^n)) - \partial_1^+(\frac{1}{\eta} v^n)\\
         g_2^n-\partial_2^+((c-\frac{1}{\eta^2})(\partial_1^- g_1^n + \partial_2^- g_2^n)) - \partial_2^+(\frac{1}{\eta} v^n)
        \end{bmatrix},
    \label{eq.split1.g}
\end{align*}
where $\partial_k^+$ and  $\partial_k^-$ denote the forward and backward difference of the $k$-th coordinate, respectively.  After getting $\bg^{n+1/2}$, we update $v^{n+1/2}$ via $\frac{1}{\eta}\nabla\cdot \bg^{n+1/2}$.

\subsection{On the solution to (\ref{eq.split.dis.2})}
By~\eqref{eq.split.dis.2}, $(u^{n+1},v^{n+1})$ solve
\begin{align}
\begin{cases}
    (u^{n+1},v^{n+1})=&\displaystyle\min_{u, v} \frac{1}{2}\int_{\Omega} |\nabla u-\bp^{n+1/2}|^2dx + \frac{1}{2}\int_{\Omega} |v-v^{n+1/2}|^2dx \\
    &+ \frac{\Delta t}{2\theta} \int_{\Omega}|u+v-f|^2dx,\\
    \bp^{n+1}=\nabla u^{n+1}.
\end{cases}
\end{align}

The optimality condition is
\begin{align}
    \begin{cases}
        \frac{\Delta t}{\theta}u^{n+1} -\nabla^2u^{n+1}+\frac{\Delta t}{\theta}v^{n+1}=-\nabla\cdot \bp^{n+1/2}+\frac{\Delta t}{\theta}f,\\
        (+ \frac{\Delta t}{\theta}) v^{n+1}+\frac{\Delta t}{\theta}u^{n+1}=  v^{n+1/2}+\frac{\Delta t}{\theta} f.
    \end{cases}
\end{align}
In a matrix form, we have $\overline{\bA}[
        u^{n+1}, v^{n+1} 
    ]^{\top}=\overline{\bb}$ with
\begin{align*}
    &\overline{\bA}=\begin{bmatrix}
        \frac{\Delta t}{\theta}-(\partial_1^-\partial_1^++ \partial_2^-\partial_2^+)& \frac{\Delta t}{\theta}\\
        \frac{\Delta t}{\theta} &+ \frac{\Delta t}{\theta}
    \end{bmatrix}, \overline{\bb}=\begin{bmatrix}
        -(\partial_1^- p_1^{n+1/2} + \partial_2^- p_2^{n+1/2})+\frac{\Delta t}{\theta} f\\
         v^{n+1/2}+\frac{\Delta t}{\theta} f
    \end{bmatrix}.
\end{align*}
\section{Numerical Experiments}\label{sec_experiment}

\begin{figure}[t!]
    \centering
    \begin{tabular}{ccccc}
    \includegraphics[width=0.16\linewidth]{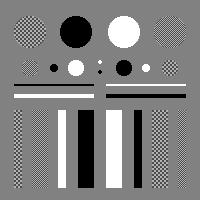} &
    \includegraphics[width=0.16\linewidth]{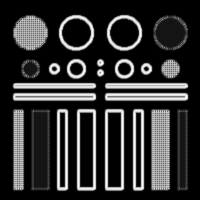} & \includegraphics[width=0.16\linewidth]{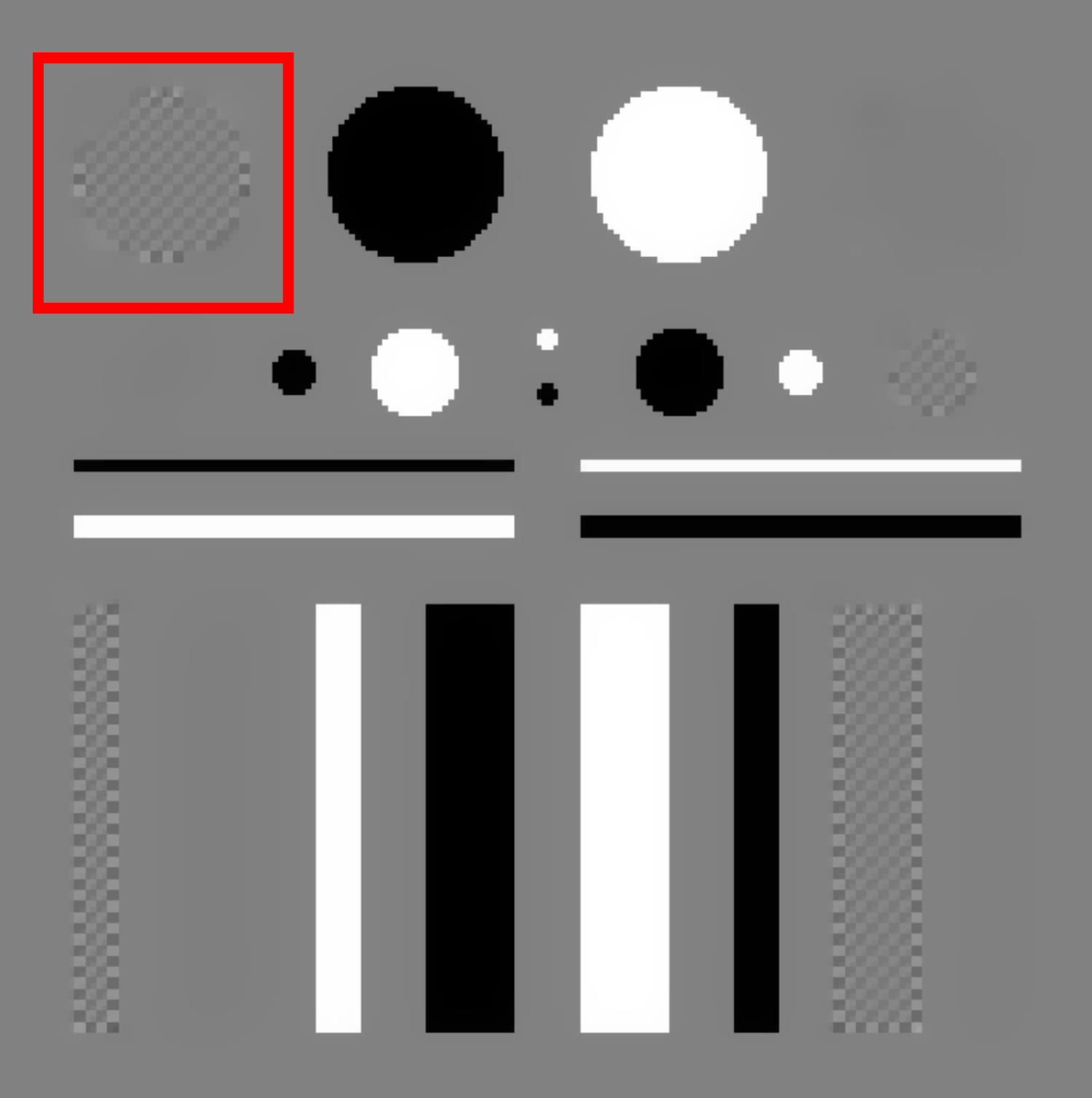} & 
    \includegraphics[width=0.16\linewidth]{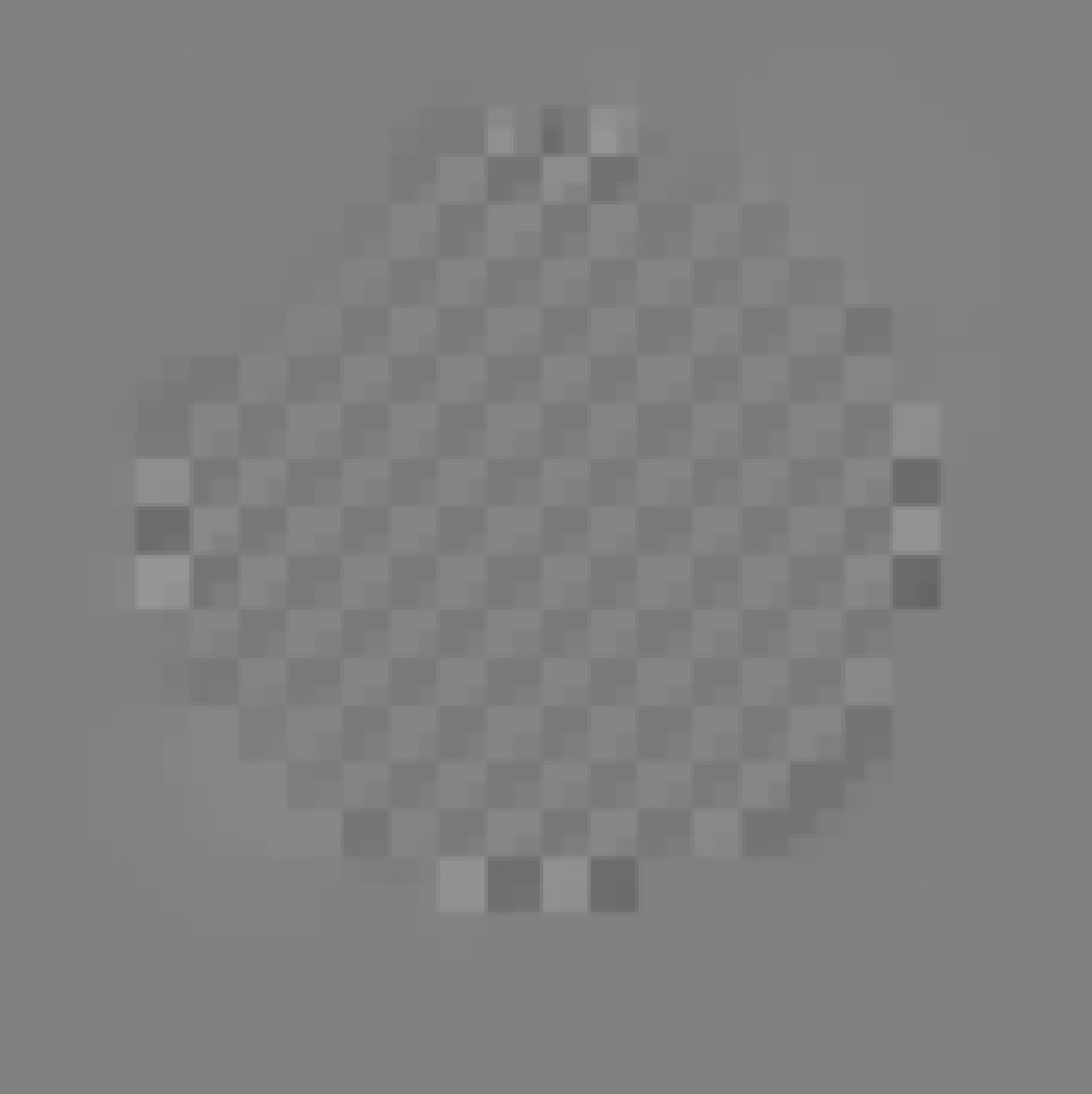}& 
    \includegraphics[width=0.16\linewidth]{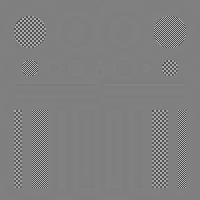} \\
    (a)&(b)&(c)&(d)&(e)\\
    \includegraphics[width=0.16\linewidth]{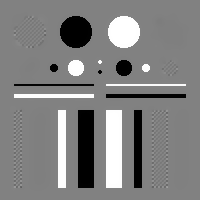} &
    \includegraphics[width=0.16\linewidth]{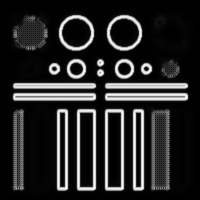} & \includegraphics[width=0.16\linewidth]{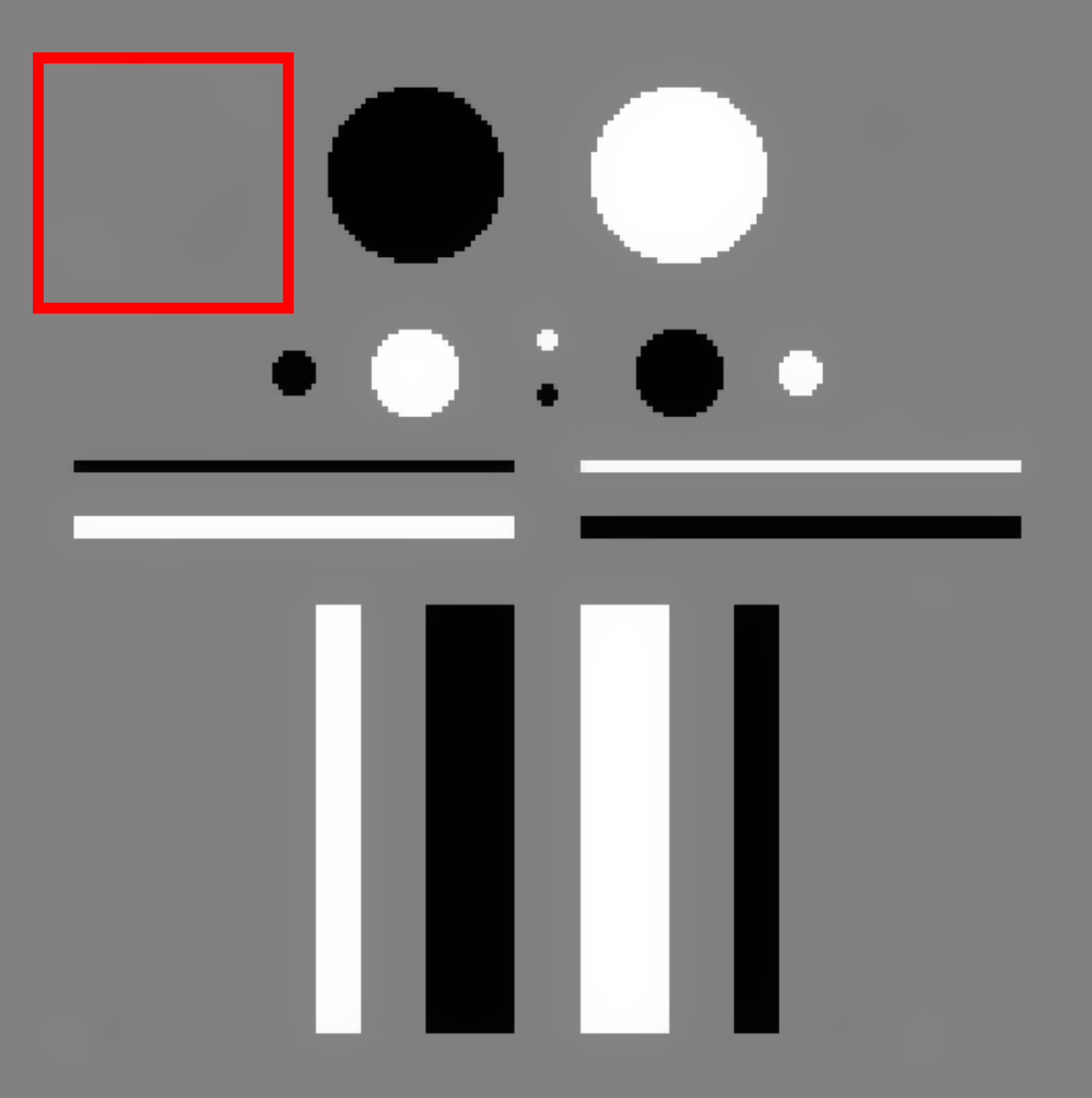} & 
    \includegraphics[width=0.16\linewidth]{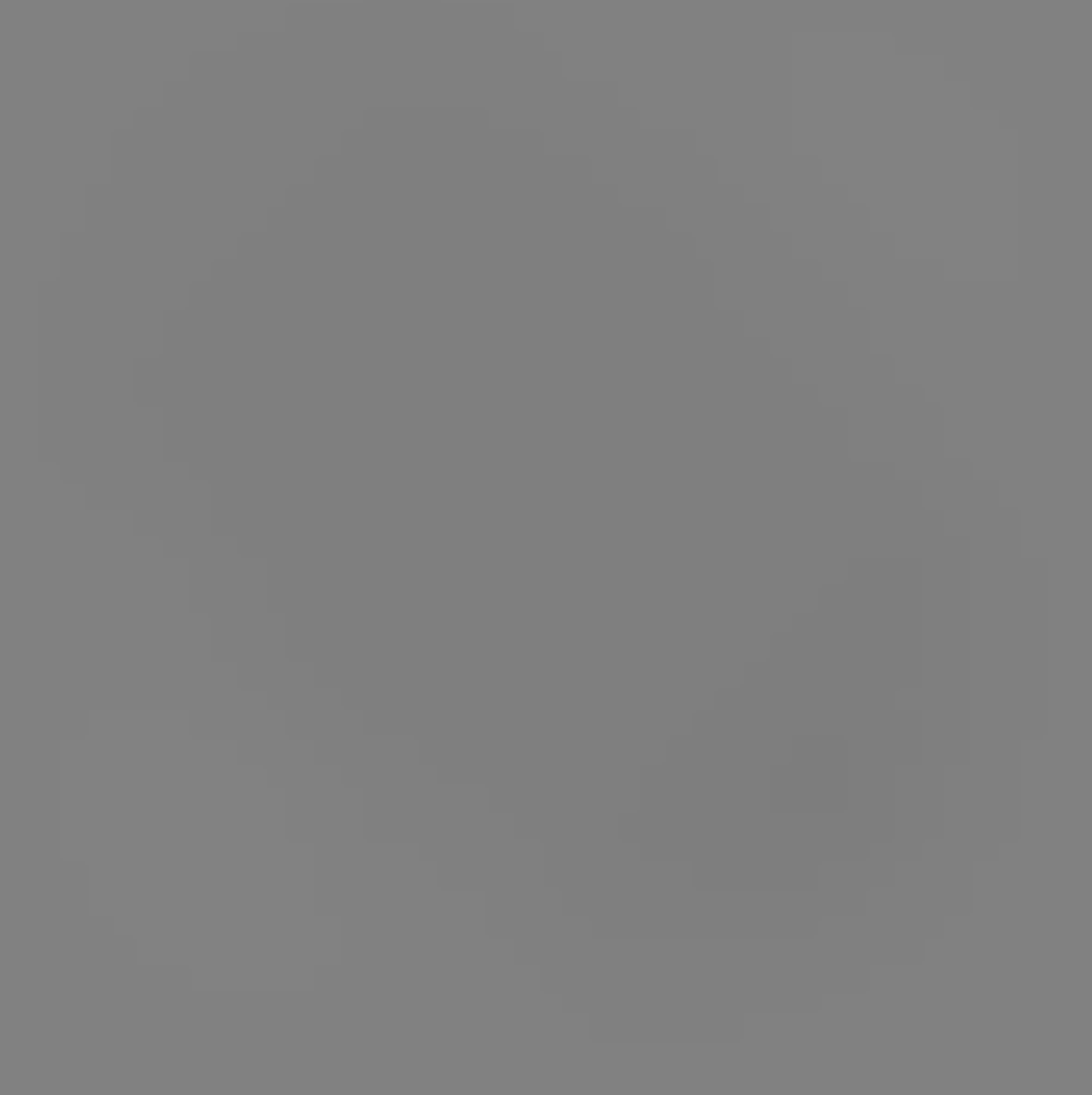} &
    \includegraphics[width=0.16\linewidth]{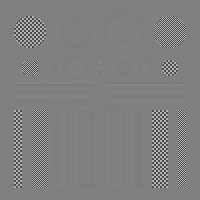} \\
    (f)&(g)&(h)&(i)&(j)\\
    \end{tabular}
    \caption{Re-starting technique. 
    (a) Input image $u^{(0)}$. (b) Normalized weight $\eta(u^{(0)})$ weight function. (c)  $u^{(1)}$  component. (d) Zoom-in of (c). (e) $v^{(1)}$ component.
    Restart the algorithm to decompose (f) $u^{(1)}$ which is identical to (c). (g) $\eta(u^{(1)})$ weight function. (h) $u^{(2)}$  component. (i) Zoom-in of (h). (j) The accumulated texture part $v^{(1)}+v^{(2)}$. Here we set $\alpha_1=0.03, \alpha_2 = 0.3, \theta = 1\times10^{-6}$, $\sigma_1=2.75$, $\sigma_2=0.75$, and $\kappa = 0.1$.}
    \label{fig:ex_restart}
\end{figure}

In this section, we  test our  model by fixing the parameters $\Delta t = 0.08$. The frozen coefficient in~\eqref{eq_frozen} is $c= 1.0$, and the maximal number of iterations is set to be $2000$. For the weight function $\eta$ in~\eqref{var_model}, 
we compute the TSV~\eqref{eq_TSV} via \eqref{eq:abcd}--\eqref{eq:Dwf_disc} which evaluate~\eqref{eq_TSV} along the horizontal, vertical, and two diagonal directions over a window of size $20$ and the standard deviation $\sigma_1, \sigma_2 >0$,
which can vary depending on the scales of textures. 
Note that $\eta$ is pre-computed and fixed throughout the iterations. 
We apply the non-local means~\cite{buades2011non} to denoise the input image for computing its TSV. \\

\noindent\textbf{Re-starting technique.} We find that applying the following re-starting technique,  similarly considered in~\cite{tadmor2004multiscale}, renders more visually appealing results. Let $u^{(0)}$ be a given image. For $k=1,2,\dots$, we solve~\eqref{eq.model.new} with $f=u^{(k-1)}$ and  update the weight with $\eta(u^{(k-1)})$. We denote the associated solution as $(u^{(k)},v^{(k)})$. After the $k$-th step, the structural part of  the input $u^{(0)}$ is  represented by $u^{(k)}$, and its textural part is represented by the accumulation $\sum_{i=1}^k v^{(i)}$. We illustrate in Figure~\ref{fig:ex_restart} the re-starting technique applied to our method for a synthetic image reported in (a), and the associated weight function is shown in (b). The obtained decomposition has residual textures present in the structure component $u^{(1)}$ shown in (c), whose region in the red box is zoomed in (d); and the identified textural component $v^{(1)}$ is in (e).
By applying the re-starting to $u^{(1)}$, whose weight function is in (g), the resulting structure  $u^{(2)}$ is reported in (h). Compared to (d), the zoom in of (h) in (i) is flat. In (j), we show the summation of $v^{(1)}$ and $v^{(2)}$ accumulated from all the  restarts. Based on this observation, in the following, we keep this restarting technique every $400$ iterations.\\
\begin{figure}[t!]
\centering
\includegraphics[width=0.75\textwidth]{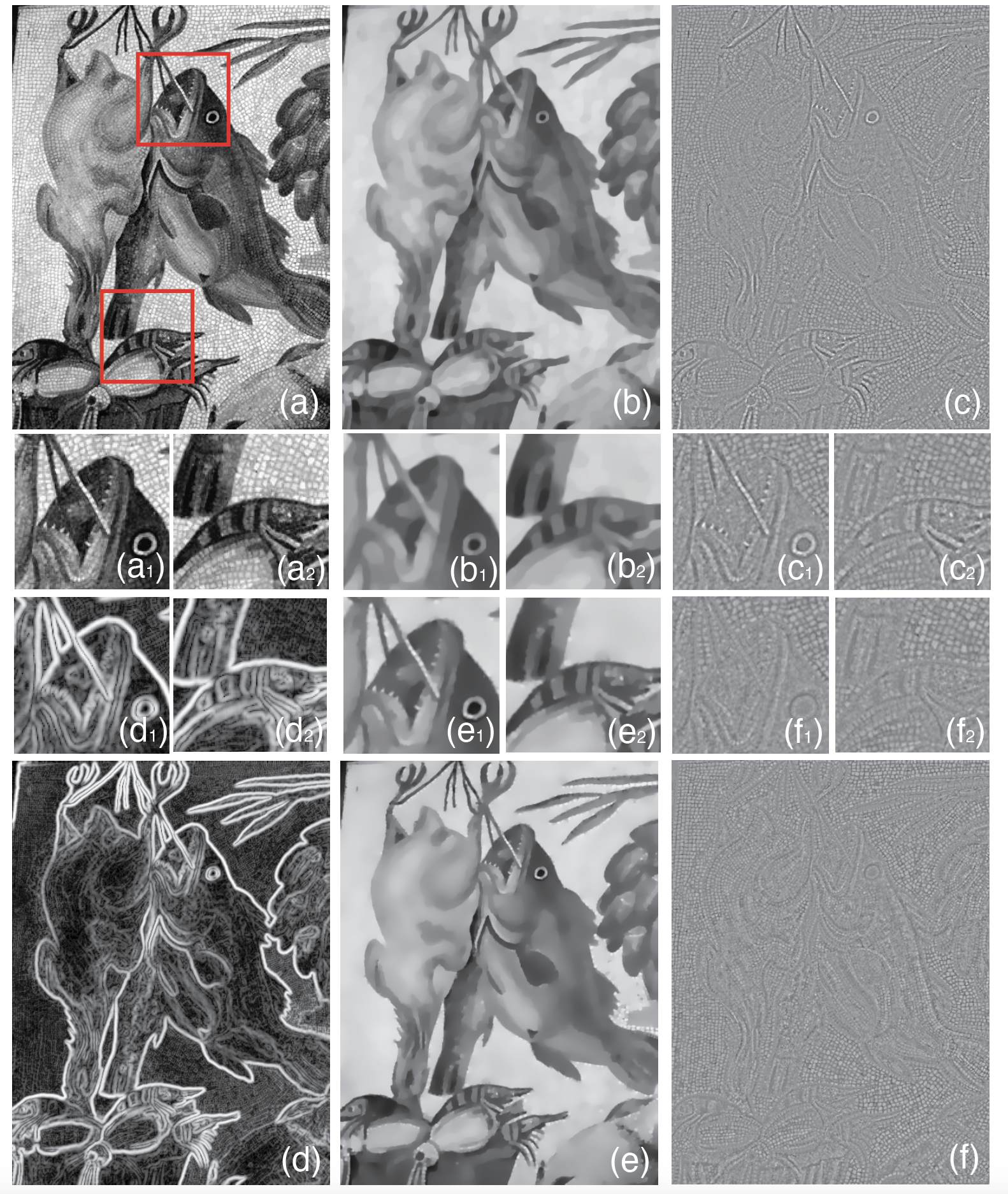}
\caption{Results by proposed method for a mosaic picture (a) and its zoom-ins in red boxes (a$_1$) and (a$_2$); normalized scale $\eta$ (d), (d$_1$), (d$_2$); resulting component $u$ (b), (b$_1$), (b$_2$), and the texture component (c), (c$_1$), (c$_2$) when using constant $\eta$ values reproducing the model in \cite{osher2003image}.
The corresponding component $u$ (e), (e$_1$), (e$_2$) and texture (f), (f$_1$), (f$_2$) are obtained using $\alpha_1=0.09, \alpha_2 = 0.01, \theta = 1\times10^{-6}$, $\sigma_1=1.5$, $\sigma_1=0.1$, and $\kappa = 0.01$. These parameters vary depending on the image sizes and texture scales.}
\label{fig:ex_mosaic}
\end{figure}

\noindent\textbf{Tests with real images.} We have tested our proposed method on different test images varying in terms of dimensions, texture type and level of detail.
In Figure~\ref{fig:ex_mosaic}, we report a particularly difficult image of mosaics in which the texture edge -- grout between individual tiles directly overlaps with the structure -- shape created with different graylevels of tiles (a).
Our proposed method proved to be robust enough to generate a suitable space-variant map for the model in (d), producing the structure in (e) and the texture components in (f).
In Figure~\ref{fig:ex_mosaic} (b) and (c) we report the resulting structure and texture component for constant $\eta$ values, which corresponds to the model of \cite{osher2003image}. \\

\begin{figure}[t!]
    \centering
    \begin{tabular}{ccccc}
    \includegraphics[width=0.19\linewidth]{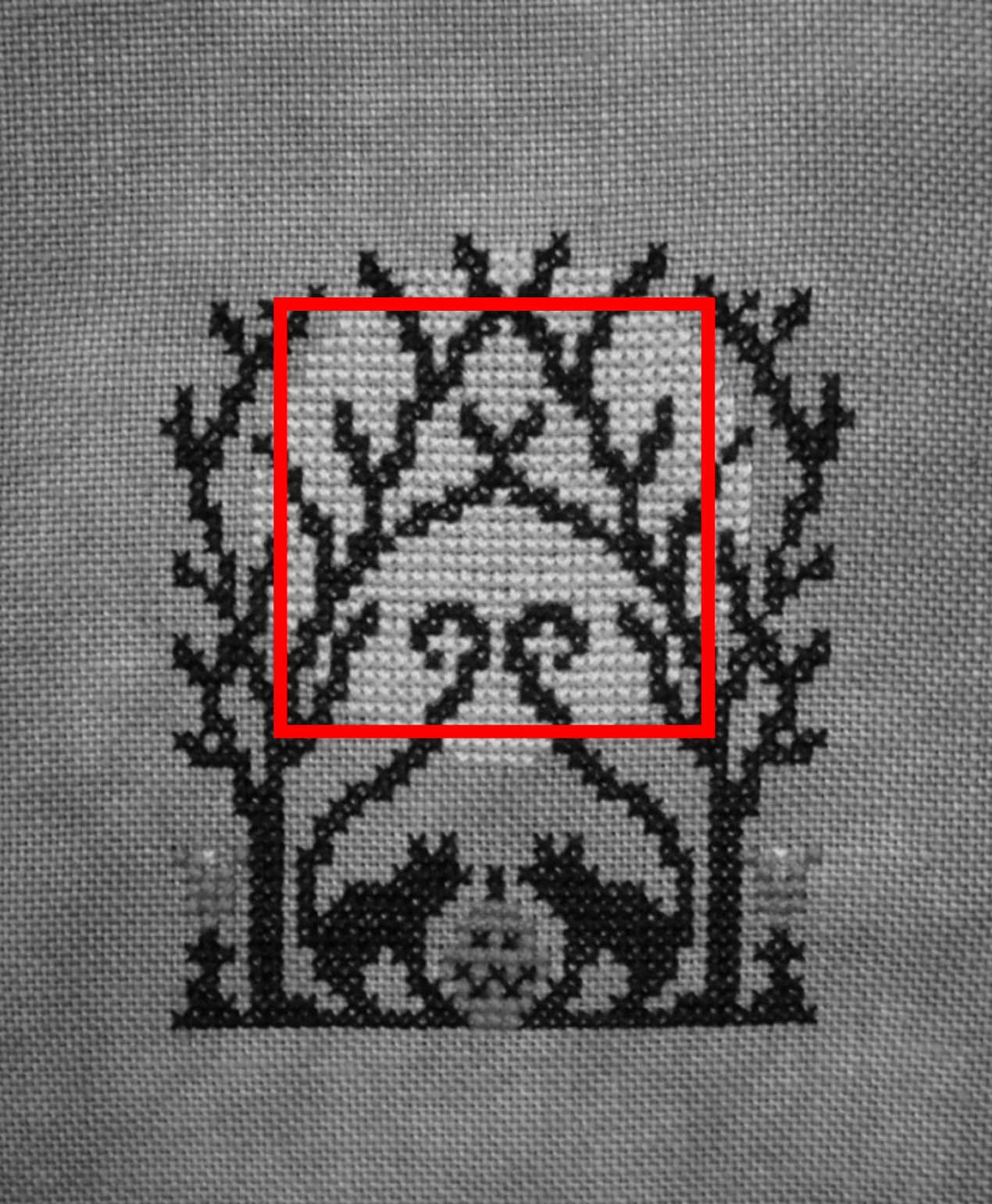} &
    \includegraphics[width=0.19\linewidth]{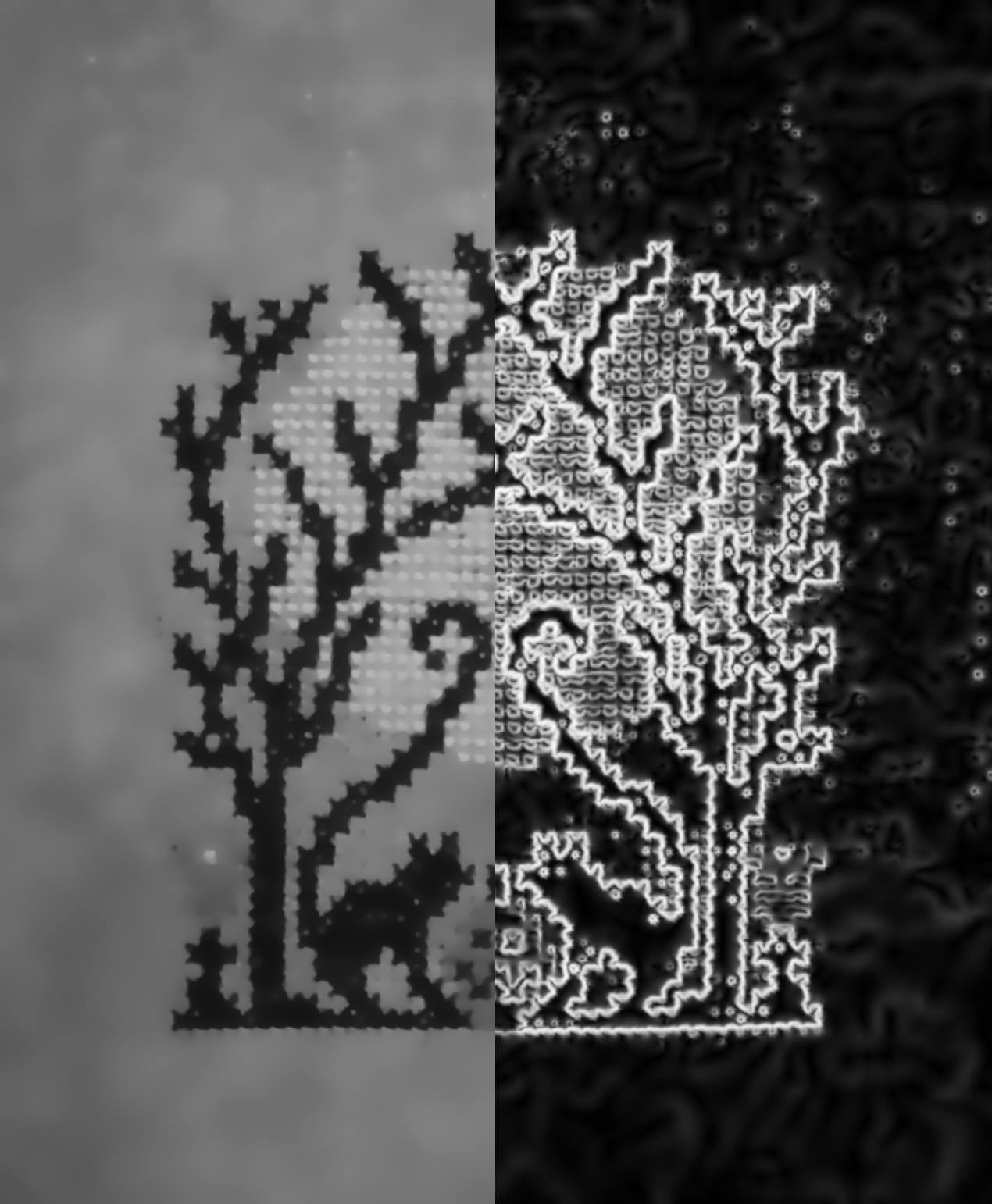} & \includegraphics[width=0.19\linewidth]{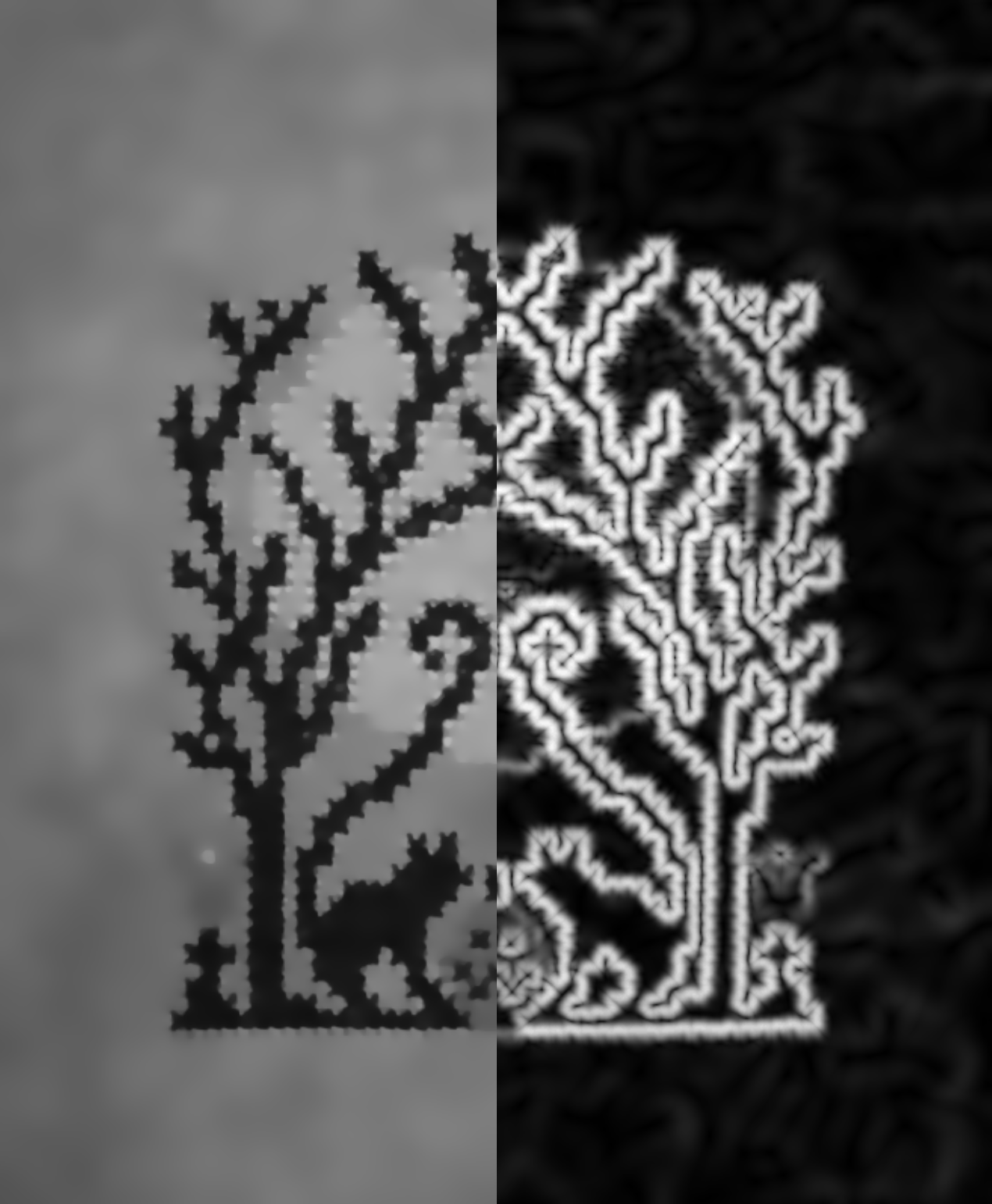} & 
    \includegraphics[width=0.19\linewidth]{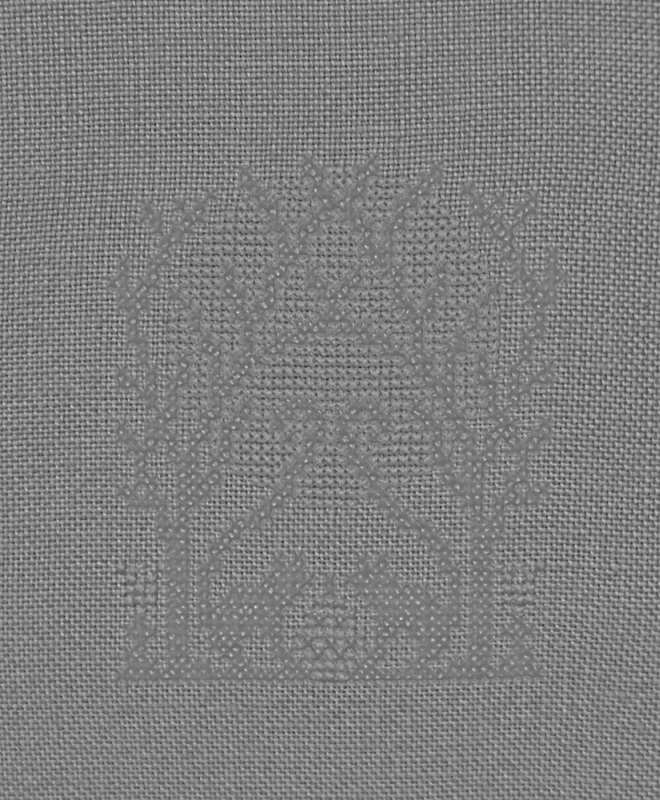}&
    \includegraphics[width=0.19\linewidth]{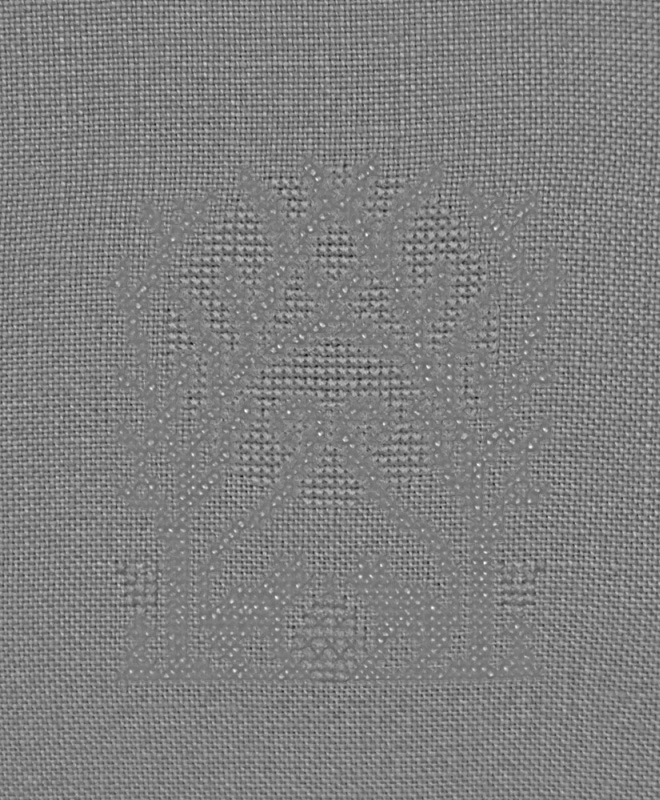}
    \\
 \includegraphics[width=0.19\linewidth]{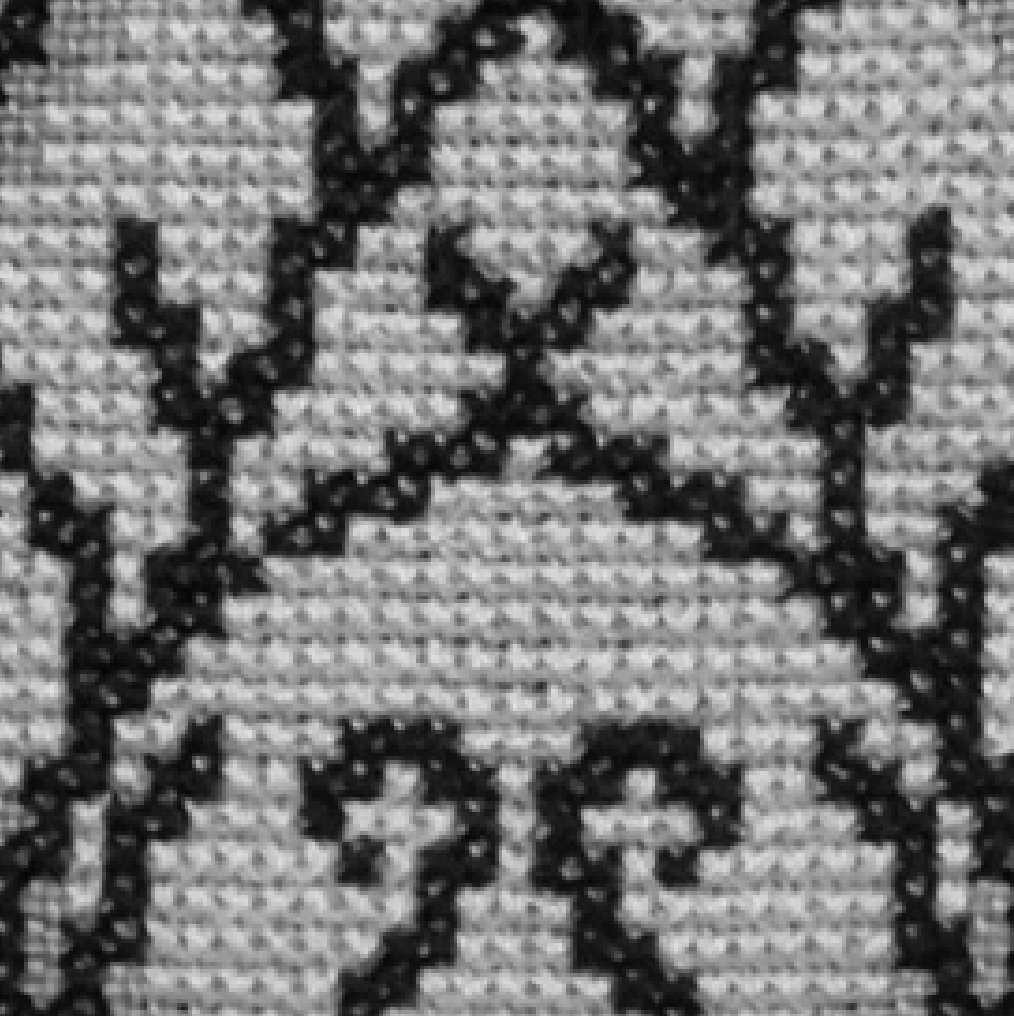} &
    \includegraphics[width=0.19\linewidth]{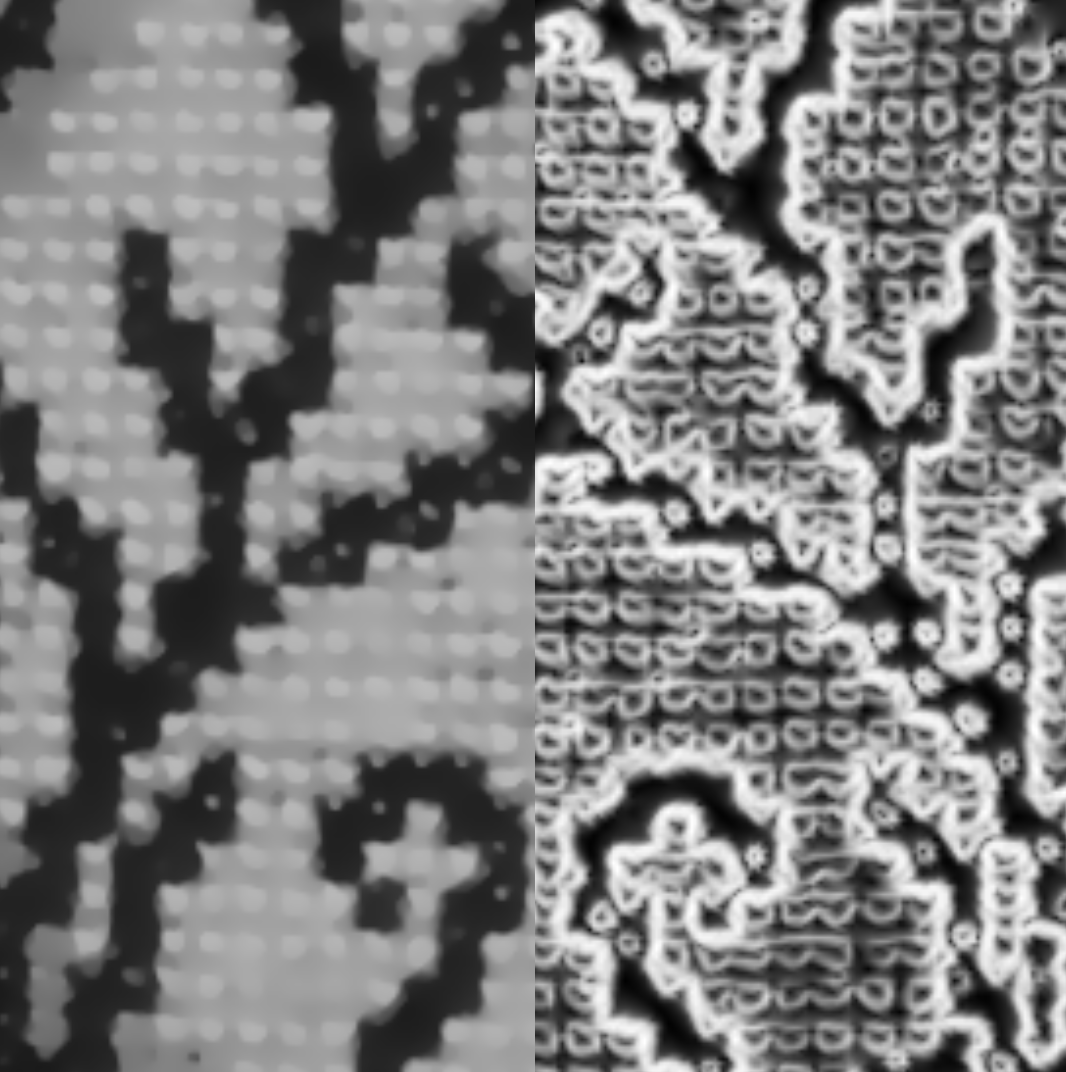} & \includegraphics[width=0.19\linewidth]{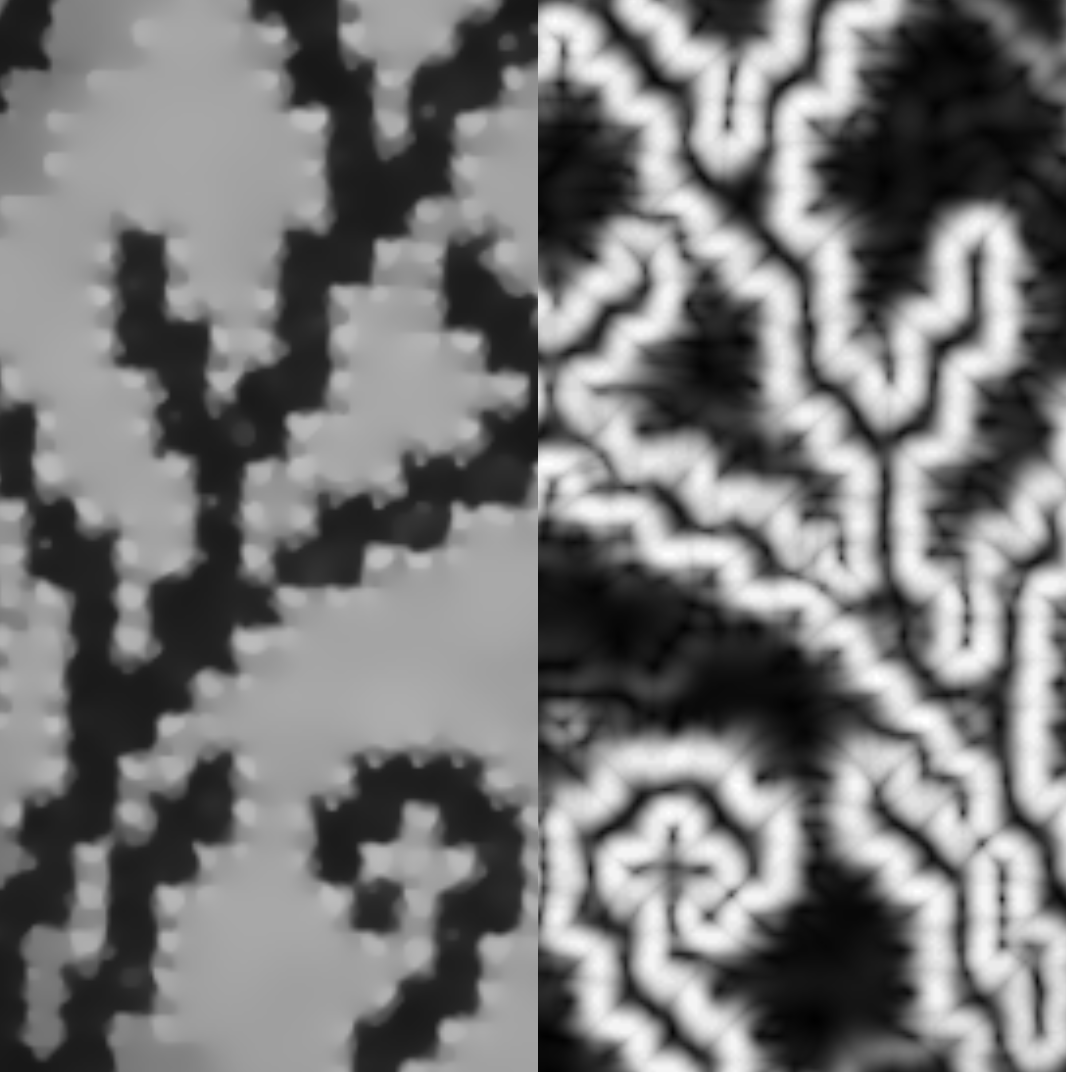} & 
    \includegraphics[width=0.19\linewidth]{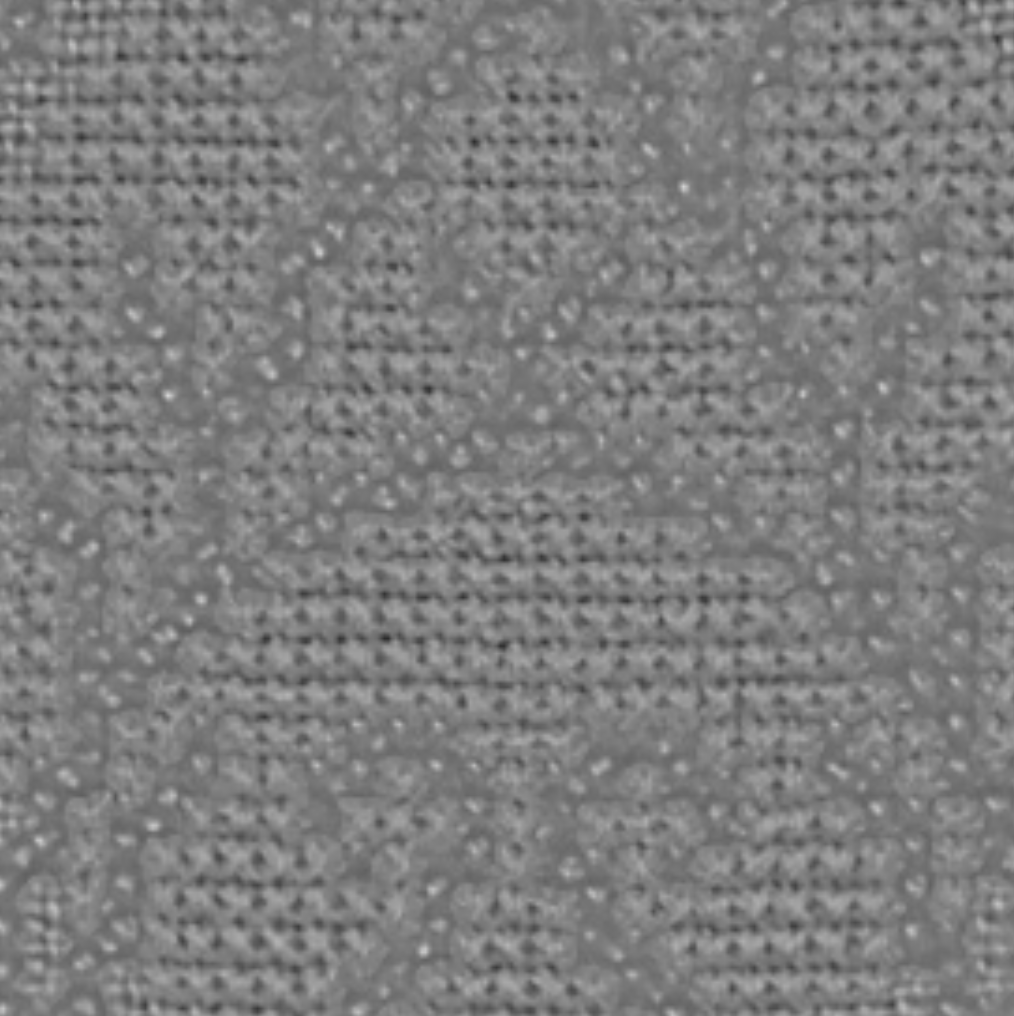} &
    \includegraphics[width=0.19\linewidth]{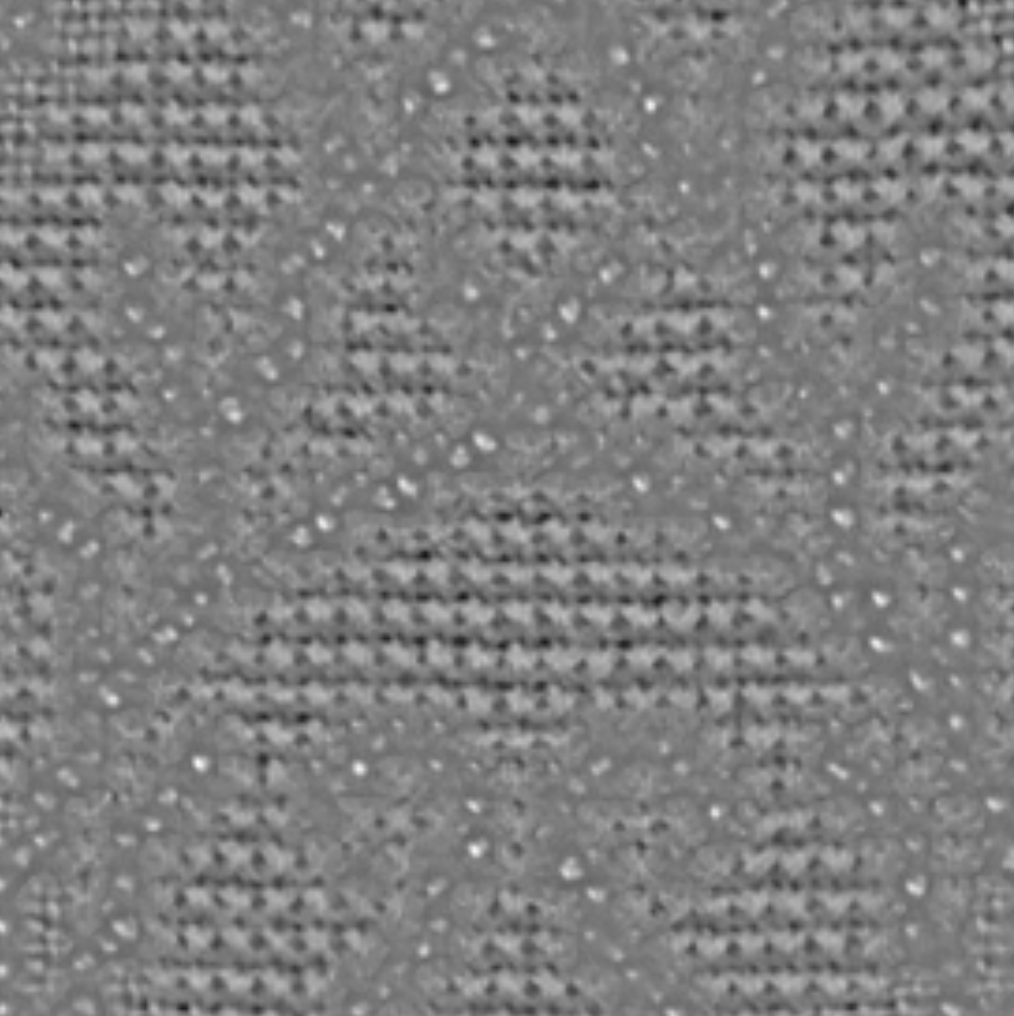} 
    \\
    (a) & (b) & (c) & (d) & (e)
    \end{tabular}
    \caption{Effects of the TSV parameter $\sigma$. Input image (a);  resulting component $u$ (left half) and normalized $\eta$ (right half) in (b) and texture component in (d) when using $\sigma_1=1.5$ and $\sigma_2=0.1$;  the corresponding results when using $\sigma_1=1.5$ and $\sigma_2=2$ are in (c) and (e). The second row shows the zoom-ins of the first row in the red box. The other parameters are $\alpha_1=0.15, \alpha_2 = 0.01, \theta = 1\times10^{-6}$,  and $\kappa = 0.01$.}
    \label{fig:diff_sigma}
\end{figure}

\noindent\textbf{Effects of $\sigma_1,\sigma_2$ on scale selection.} In the next example, we notice that the TSV discretization in \eqref{eq:abcd}--\eqref{eq:Dwf_disc} is directly influenced by two factors: the window size or the line length $\sigma_1$ and the line width $\sigma_2$.
In Figure~\ref{fig:diff_sigma}, the input image (a) presents two different texture scales: the canvas and the stitches.
We computed the space-variant normalized weight field $\eta$ (b)-(c) using $\sigma_1=1.5$, $\sigma_2=0.1$ and $\sigma_1=1.5$, $\sigma_2=2$, respectively.
The smaller $\sigma_2$ value clearly indicates the stitches as structural edges, consequently maintained in the structure component $u$ (b), while the evaluation of $\eta$ over a wider line corresponding to the larger $\sigma_2$ value in Figure \ref{fig:diff_sigma} (c) results in capturing the intensity as a structure component $u$ (c) while the stitches were captured in the texture component (e). The zoom-ins are reported in the second row of Fig. \ref{fig:diff_sigma}.

\section{Conclusion}\label{sec_conclude}
We introduced a novel image decomposition model using total symmetric variation (TSV) to detect local symmetric intensity variations. The model decomposes an image into its structure component regularized by total variation and  texture component  by a weighted 
$G$-norm, with weights derived from TSV. It ensures to capture only uniform textural interiors while excluding boundaries. The method extends to color images and can be enhanced with a multiscale strategy. Experiments show the model's effectiveness, yielding promising results.

\section*{Acknowledgement}
Roy Y. He  was partially supported by CityU grant 7200779. M. Huska was partially supported by INdAM - GNCS Project CUP\_E53C24001950001 and this study was carried out within the MICS (Made in Italy - Circular and Sustainable) Extended Partnership with received funding from the European Union Next-GenerationEU (PNRR) - D.D. 1551.11-10-2022, PE00000004). H. Liu was partially supported by National Natural Science Foundation of China grant 12201530 and HKRGC ECS grant 22302123.

\bibliographystyle{splncs04}
\bibliography{mybibliography}
\end{document}